\documentclass{article}

\usepackage[preprint]{sty/neurips_2021}

\usepackage[utf8]{inputenc} 
\usepackage[T1]{fontenc}    
\usepackage{hyperref}       
\usepackage{url}            
\usepackage{booktabs}       
\usepackage{amsfonts}       
\usepackage{nicefrac}       
\usepackage{microtype}      
\usepackage[table]{xcolor}         
\usepackage{sty/notations}
\usepackage{sty/dm-colors}
\usepackage{multirow}
\usepackage{threeparttable} 
\usepackage[linesnumbered,ruled,vlined]{algorithm2e}
\hypersetup{
	colorlinks,
	citecolor=dmblue300,
	linkcolor=dmred300,
	urlcolor=dmorange500}

\usepackage{amsmath,amssymb,amsthm,bm,mleftright}

\usepackage{xcolor}
\definecolor{LightGray}{gray}{0.9}

\definecolor{misocolor}{rgb}{0.8, 0.33, 0.00}

\definecolor{pierrecolor}{RGB}{ 31, 156, 253}

\newcommand{\Dict}{\algofont{\texttt{Dict}}\xspace}
\newcommand{\List}{\algofont{\texttt{List}}\xspace}
\newcommand{\Set}{\algofont{\texttt{Set}}\xspace}
\newcommand{\policyalgo}{\hyperref[alg:policy]{\texttt{MaxPlayerPolicy}}\xspace}

\newcommand{\efficientouralgo}{\hyperref[alg:memory efficient ouralgo]{\texttt{PracticalIXOMD}}\xspace}
\newcommand{\algofont}[1]{\normalfont\texttt{#1}}

\newcommand{\I}{\mathbb{I}}

\newcommand{\poly}{\mathrm{poly}}

\newtheorem{theorem}{Theorem}
\newtheorem{corollary}{Corollary}[theorem]
\newtheorem{lemma}[theorem]{Lemma}

\newtheorem{remark}{Remark}

\let\Pr\relax
\newcommand{\Pr}{\bm{\mathrm{Pr}}}
\newcommand{\D}{\mathrm{D}}

\newcommand{\regret}{\mathfrak{R}}

\newcommand{\maxpi}{\Pi_{\mathrm{max}}}

\newcommand{\minpi}{\Pi_{\mathrm{min}}}

\newcommand{\N}{\mathbb{N}}

\newcommand{\angl}[1]{\mleft\langle #1 \mright\rangle}
\newcommand{\pp}[1]{\mleft( #1 \mright)}
\let\bb\relax
\newcommand{\bb}[1]{\mleft[ #1 \mright]}

\newcommand{\mbb}{\middle \|}

\newcommand{\pluseq}{\mathrel{+}=}

\newcommand{\muring}{\mathring{\mu}}

\def\cA{{\mathcal{A}}}
\def\cB{{\mathcal{B}}}

\def\cF{{\mathcal{F}}}

\def\cL{{\mathcal{L}}}

\def\cO{{\mathcal{O}}}

\def\cS{{\mathcal{S}}}
\def\cT{{\mathcal{T}}}

\def\cX{{\mathcal{X}}}
\def\cY{{\mathcal{Y}}}

\def\fL{\mathfrak{L}}

\author{%
  Tadashi Kozuno\thanks{Equal contribution}\\
  University of Alberta\\
  \texttt{tadashi.kozuno@gmail.com}\\
  \And
  Pierre M{\'e}nard\footnotemark[1]\\
  Otto von Guericke Universit{\"a}t Magdeburg\\
  \texttt{pierre.menard@ovgu.de}\\
  \AND
  R{\'e}mi Munos\\
  DeepMind Paris\\
  \texttt{munos@deepmind.com}\\
  \And
  Michal Valko\\
  DeepMind Paris\\
  \texttt{valkom@deepmind.com}\\
}

\title{Model-Free Learning for\\ Two-Player Zero-Sum Partially Observable\\ Markov Games with Perfect Recall}

\begin{document}

\maketitle

\begin{abstract}
We study the problem of learning a Nash equilibrium (NE) in an imperfect information game (IIG) through self-play. Precisely, we focus on two-player, zero-sum, episodic, tabular IIG under the \textit{perfect-recall} assumption where the only feedback is realizations of the game (bandit feedback). In particular, the \textit{dynamics of the IIG is not known}---we can only access it by sampling or interacting with a game simulator. For this learning setting, we provide the Implicit Exploration Online Mirror Descent (\ouralgo) algorithm. It is a model-free algorithm with a high-probability bound on the convergence rate to the NE of order $1/\sqrt{T}$ where~$T$ is the number of played games. Moreover, \ouralgo is computationally efficient as it needs to perform the \textit{updates only along the sampled trajectory}.
\end{abstract}

\section{Introduction}

We study the setting of \emph{learning} a Nash equilibrium (NE, \citealp{nash1950equilibrium}) in an \textit{imperfect information game} (IIG, \citealp{osborne1994course}). Precisely, we focus on two-player zero-sum IIG under the \textit{perfect-recall} assumption \citep{kuhn1953extensive}. Perfect recall means that the players \textit{do not forget} observations encountered or actions taken during the game. We model the game as a tabular, episodic (of horizon $H$, \textit{partially observable Markov game} (POMG) with a state space of size $S$, action spaces of size $A$ and $B$ for the max- and min-player respectively, and observation spaces (i.e., information set spaces, which are partitions of the state space) of size $X$ and $Y$ for the max- and min-player. In learning by \textit{self play}, we control \textit{both} the max \emph{and} min-player. After $T$ episodes of the game we are asked to return a profile that is close to a NE in terms of \textit{exploitability} gap \citep{ponsen2011computing}.

\paragraph{Full feedback} In case when we have perfect knowledge of the game (i.e., the transition probabilities and rewards) there already exist several methods approximating the NE. The first line of work casts the setting through the sequence-form representation as a linear program which can be solved efficiently for games with moderate sizes of observation spaces~$X$ and~$Y$  \citep{Rom62,VONSTENGEL1996220,koller1996efficient}. 
The sequence-from representation allows also to cast the setting as \textit{finding a saddle point} \citep{hoda2010smoothing}. It is then possible to adapt first-order methods such as Nesterov’s smoothing \citep{nesterov2005smooth} and \MirrorProx \citep{nemirovski2004prox} to IIG, as done respectively by \citet{hoda2010smoothing,kroer2018solving} and \citet{kroer2015faster,kroer2020faster}. These methods have a rate of convergence of order $\tcO((X+Y)/T)$, where $\tcO$ hides poly-log terms in $e^H,X,A,Y,B,T$.\footnote{Therefore, we hide \emph{polynomial} dependence on the horizon $H$.} Note that \emph{game-dependent} exponential rate could also be obtained with first-order methods, see \citet{gilpin2012first} and \citet{munos2020fast}. Another important line of work relies on minimizing the \textit{counterfactual regret} \citep{zinkevich2007regret}. It uses an algorithm designed for adversarial bandits to locally minimize the regret of each player. A well-known example is \CFR by \citet{zinkevich2007regret} based on the regret-matching algorithm \citep{hart2000simple,gordon2007no}. There exist many other variants of it, such as \CFRp \citep{tammelin2014solving, burch2019revisiting}, see also \citet{ farina2019optimistic,farina2020faster}. These algorithms however only enjoy a (known) guarantee of convergence of order $\tcO((X\sqrt{A}+Y\sqrt{B})/\sqrt{T})$.
Note that the two last approaches require computing a full feedback: either some gradient for the first-order methods or the local regret for counterfactual regret minimization.
Usually, this can be done by a complete traversal of the state space leading to a time-complexity of order $\cO(S)$. 
Sampling can reduce this time-complexity to $\cO(X+Y),$\footnote{Note that $\cO(X+Y)$ is at most $\cO(S)$.} i.e., we sample the transitions and the actions of the other player; see for example the external-sampling \MCCFR algorithm \citep{lanctot2009monte, farina2020stochastic}.

\paragraph{Bandit feedback} In this paper, we consider a more challenging setting where we \emph{only observe realizations of the games (bandit feedback) and do not have any prior knowledge of the game}. Precisely, the rewards, the transition probabilities (sometimes modeled as the policy of a \textit{chance player}), the observation/state space, and its (tree) structure are unknown.

\paragraph{Bandit feedback, model-based} To deal with the limited bandit feedback, \citet{zhou2020posterior} consider model-based approach by using \textit{posterior sampling} (PS, \citealp{strens2000bayesian}) to learn a model and then use the \CFR algorithm in games sampled from the posterior. They obtain a convergence rate of order $\tcO(\max(XA+YB,\sqrt{S})/\sqrt{T})$ but only when the games are actually sampled according to the known prior. In addition, they still need to know the state
space and its structure\footnote{ \label{note:structure}By \textit{structure} we refer to the tree structure of the state space or observations spaces, see Section~\ref{sec:preliminary}.} in order to instantiate the prior. Instead, \citet{zhang2020finding} rely on the principle of optimism in presence of uncertainty to incrementally build a model of the game. Then, they feed optimistic local regrets to a counterfactual regret minimizer algorithm such as the \CFR algorithm. They prove a high-probability bound on the exploitability gap of order $\tcO((X\sqrt{A}+Y\sqrt{B})/\sqrt{T})$.

\paragraph{Bandit feedback, model-free} Our results follows another line of work which consider a \emph{model-free} approach. A well known algorithm of this type is outcome-sampling \MCCFR \citep{lanctot2009monte, farina2020stochastic}, which builds an importance-sampling estimate of the counterfactual regret given \emph{exploration profile} (named balanced strategy by \citealp{farina2020stochastic}).
This exploration strategy should ensure that the players explore the information sets uniformly (i.e., such that all induced reach probabilities are lower-bounded by an absolute constant). Note that it is not clear how to find such an exploration profile without knowing the structure of the game.\textsuperscript{\ref{note:structure}} In particular, following the uniform distribution over the actions at each information set is not necessarily a good choice, e.g., when the tree formed by the information set space is not balanced. This algorithm has a guarantee of order $\tcO((X\sqrt{A}+Y\sqrt{B})/\sqrt{T})$ with high probability. Building on this idea, \citet{farina2021model-free} propose to mix the exploration profile with one produced by a counterfactual regret minimizer such as \CFR. They prove a high-probability bound on the exploitability gap of order $\tcO(\poly(X,A,Y,B)/T^{1/4})$.  Note that this bound is a consequence of a bound on the regret of both players (see Section~\ref{sec:preliminary}) that holds even in the non-stochastic setting where an adversary picks a new game at each episode. Closer to our approach, \citet{farina2021bandit} recast the setting to an adversarial bandit linear optimization  (\citealp{flaxman2005online,Abernethy2008}, see also Section~\ref{sec:conversion_to_online}). Precisely, they use the online mirror descent (\OMD) algorithm with the dilated entropy distance-generating function \citep{ hoda2010smoothing,kroer2015faster} as regularizer. Then,  \OMD  is fed with an estimate of the losses of the reformulated adversarial bandit linear instance. The estimator is a generalization of the typical one-point linear regression \citep{dani2008price}. 
They obtain a rate of order $\tcO((XA+YB)/\sqrt{T})$, 
which is, similarly as done by \citet{farina2021model-free}, derived from a regret bound valid in the adversarial setting. However, their bound holds only in expectation and not in high probability.

To obtain high-probability bound, we instead propose to use an importance sampling estimator of the losses with \emph{implicit exploration} \citep{kocak2014efficient,neu2015explore}. Indeed, the implicit bias of this estimator allows to effortlessly control the variance of the estimate, see \citet[Chapter 12]{lattimore2020bandit_book} for an in-depth discussion. Using this estimator, we give the Implicit Exploration Online Mirror Descent (\ouralgo) based on \OMD with the dilated entropy distance-generating function (using uniform weights) as a regularizer and add implicit exploration in the importance sampling estimator of the losses. Using our new analysis of this particular combination, we prove a high-probability bound on the exploitability gap of the average profile of order $\tcO((X\sqrt{A}+Y\sqrt{B})/\sqrt{T})$; cf.\,Table~\ref{tab:comparison_rate} to see how our result compares to the prior work mentioned above. Precisely, our bound is obtained by bounding the regret of each player if they both follow the policy prescribed by \ouralgo. Note that the regret bound, e.g., of the max-player, of order $\tcO(X\sqrt{\smash[b]{AT}})$, remains valid if the opponent's policy \emph{and} the game are picked by an adversary at each episode. \ouralgo shares some similarities with the approach of \citet{jin2020learning} designed for a different setting (see Remark~\ref{rem:difference_with_jin}). A notable difference is that we use the dilated entropy distance-generating function as a regularizer instead of  the un-normalized Kullback-Leibler divergence \citep{rosenberg2019online}. Our choice of regularizer allows an efficient update of the current policy with a $\cO(HA)$ time-complexity per episode  (see Section~\ref{sec:implementation}). In particular, our result answers the open problem raised by \citet{farina2021bandit} and \citet{farina2021model-free} of providing an algorithm with high-probability regret bound scaling with $\sqrt{T}$ with $\cO(HA)$ computations per episode. Interestingly, we can also update the average profile (which will be returned at the end of the learning, see Section~\ref{sec:implementation}) in an online fashion. As consequence, \ouralgo enjoys an overall time-complexity of  $\cO(TH(A+B) + \min(TH,X)A+  \min(TH,Y)B)$ and space-complexity of order $\cO(\min(TH,X)A+ \min(TH,Y)B)$. 

Moreover, \ouralgo requires almost no prior knowledge of the game. In particular, we do not need to know  the list of information sets in advance. We only require an oracle providing the possible actions at encountered information sets and a bound on $A$, $B$, and $H$ to  optimally\footnote{Precisely, with this knowledge we obtain a regret bound, e.g. for the max-player, of order $\tcO(X\sqrt{AT})$; whereas we get $\tcO(XA\sqrt{T})$ without it.} tune the learning rate, see Remark~\ref{rem:tuning_lr}.

\begin{table}[t!]
\centering
\begin{threeparttable}
 \resizebox{\textwidth}{!}{  
\renewcommand{\arraystretch}{1.5}
    \begin{tabular}{@{}|c|c|c|c|@{}}
    \hline
     \textbf{Algorithm} &  & \textbf{Adv.\,game}  & \textbf{Rate}\\
		\hline
		\hline
     \citet{zhou2020posterior} & \multirow{2}{*}{model-based} & \multirow{3}{*}{no} & $\tcO(\max(X \sqrt{A}+Y\sqrt{B},\sqrt{S})/\sqrt{T})$\;\tnote{1} \\
     \cline{1-1}\cline{4-4}
     \citet{zhang2020finding} &   &  & $\tcO((X \sqrt{A}+Y\sqrt{B})/\sqrt{T})$\\
     \cline{1-2}\cline{4-4}
      {\scriptsize \citet{lanctot2009monte, farina2020stochastic}} & \multirow{4}{*}{model-free} &  &  $\tcO((X \sqrt{A}+Y\sqrt{B})/\sqrt{T})$ \\
     \cline{1-1}\cline{3-4}
     \citet{farina2021model-free} &  & \multirow{3}{*}{yes} & $\tcO(\poly(X,A,Y,B)/T^{1/4})$\\
     \cline{1-1}\cline{4-4}
     \citet{farina2021bandit} &  &  & $\tcO((XA+YB)/\sqrt{T})$\;\tnote{2}\\
     \cline{1-1}\cline{4-4}
     \rowcolor{LightGray}
         \ouralgo (this paper) & &  & $\tcO((X \sqrt{A}+Y\sqrt{B})/\sqrt{T})$\\
     \hline
    \end{tabular}}
\hspace{+0.3cm}
    \caption{Algorithms for computing a NE of an IIG with bandit feedback and their respective upper bound on the exploitability gap after $T$ episodes. In the adversarial game column we precise  whether the algorithm could be used to obtain a $\sqrt{T}$-regret for one player when the other player and the game are chosen by an adversary at each episodes.}
    \vspace{-0.25cm}
    {\scriptsize
    \begin{tablenotes}
    \item[1] Only in expectation according to a known prior on the game.
    \item[2] Only in expectation.
\end{tablenotes}}
\end{threeparttable}
    \label{tab:comparison_rate}
\end{table}

We highlight our main contributions:
\begin{itemize}
    \item We give the \ouralgo algorithm that learns a NE of an IIG in self-play with limited feedback. It has a provably high-probability convergence rate of order $\tcO((X\sqrt{A}+Y\sqrt{B})/\sqrt{T})$. The time-complexity of \ouralgo is of order $\cO(TH(A+B) + \min(TH,X)A+  \min(TH,Y)B)$ with  a space-complexity of order $\cO(\min(TH,X)A+ \min(TH,Y)B)$.
    \item If only one player follows \ouralgo, e.g., the max-player, then its regret is  w.h.p.\,at most $\tcO(X\sqrt{AT}).$ The important property of our result is that it remains valid even if the policy \emph{and} the game are picked by an adversary at each episode. Furthermore, the time-complexity of \ouralgo per episode is of order $\cO(HA)$. This answers  an open problem of \citet{farina2021bandit,farina2021model-free}.
    \item \ouralgo only needs to  know the possible actions at the encountered information sets and a bound on $A$, $B$, and $H$ to tune the learning rate. In particular, we do not need to know the list of information sets in advance.
\end{itemize}

\section{Preliminaries}\label{sec:preliminary}

In this section, we introduce our notations and our setting---partially observable Markov game (POMG) with bandit feedback and perfect recall. For a positive integer $i$, we denote by $[i]$ the set $\{1, 2, \ldots, i \}$. For a finite set $\cA$, we let $\Delta_\cA$ or $\Delta (\cA)$ denote the set of all probability distributions over $\cA$.

\paragraph{Partially observable Markov game (POMG)}

We consider an episodic, tabular, two-player, zero-sum POMG $(\cS,\cX,\cY,\cA,\cB,H,\{p_h\}_{h \in [H]},\{r_h\}_{h \in [H]})$, which consists of the following components \citep{littman1994markov, shapley1953stochastic}: a finite state space~$\cS$ of size $S$, its information set spaces (partitions of $\cS$) $\cX$ of size $X$ and $\cY$ of size $Y$ for the max- and min-player (resp.), finite action spaces $\cA$ of size $A$ and $\cB$ of size $B$ for the max- and min-player (resp.), time-horizon $H \in \N$, initial state distribution $p_0 \in \Delta (\cS)$, a state-transition probability kernel $p_h: \cS \times \cA \times \cB \rightarrow \Delta(\cS)$ for each $h \in [H]$, and a reward function $r_h: \cS \times \cA \times \cB \rightarrow [0, 1]$ for each $h \in [H]$. For a state $s \in \cS$ we denote by $x(s) \in \cX$ and $y(s) \in \cY$ information sets such that $s \in x(s)$ and $y \in y(s)$.

\paragraph{Learning procedure} The players play this game for $T$ episodes, following so-called policies. A policy $\mu$ of the max-player is a sequence $(\mu_h)_{h \in [H]}$ of mappings from $\cX_h$ to $\Delta_\cA$. ($\cX_h \subset \cX$ is defined later.) A policy $\nu$ of the min-player is defined similarly. We let $\maxpi$ and $\minpi$ denote the sets of max- and min-player's policies, respectively. The $t$-th episode proceeds as follows: an initial state $s_1^t$ is sampled from $p_0$. At the step $h$, the max- and min-player (resp.) observe their information sets $x_h^t := x (s_h^t)$ and $y_h^t := y (s_h^t)$. Given the information, the max- and min-player (resp.) choose and execute actions $a_h^t \sim \mu_h^t (\cdot | x_h)$ and $b_h^t \sim \nu_h^t (\cdot | y_h)$. As a result, the current state transitions to a next state $s_{h+1}^t \sim p_h (\cdot | s_h^t, a_h^t, b_h^t)$, and the max- and min-player receive rewards $r_h^t := r_h (s_h^t, a_h^t, b_h^t)$ and $-r_h^t$, respectively. This is repeated until a time step $H$, at which the episode finishes.

\paragraph{Tree-like game structure and perfect recall assumption}

We assume that the game has a tree-like structure: for any state $s \in \cS,$ there is a unique step $h$ and history $(s_1, a_1, b_1, \ldots, s_h = s)$ to reach~$s$. Precisely, for any policy of the players, for any realization of the game (i.e., trajectory) $(s_k',a_k',b_k')_{k\in[H]}$, conditionally to $s'_{i} = s$, it almost surely holds that $i= h$ and $(s'_1,\ldots,s'_h) = (s_1,\ldots,s_h)$. We also assume perfect recall, which means that each player remembers its past observations and actions. For example, in case of the max-player, for each information set $x \in \cX$ there is a unique history $(x_1, a_1, \ldots, x_h = x)$ up to~$x$. These assumptions require that $\cX$ can be partitioned to $H$ subsets $(\cX_h)_{h \in [H]}$ such that $x_h \in \cX_h$ is reachable only at time step $h$; otherwise there would be two different histories up to $x_h$. $\cS$ and $\cY$ can be also partitioned into $H$ subsets $(\cS_h)_{h \in [H]}$, and $(\cY_h)_{h \in [H]}$, respectively.

Given the assumptions above, there exists a unique history $(s_1, a_1, b_1, \ldots, s_h=s, a_h=a, b_h=b)$ ending with $(s_h=s, a_h=a, b_h=b)$ for any state $s \in \cS_h$, the max-player's action $a \in \cA$, and the min-player's action $b \in \cB$. Accordingly, the probability of $s_h = s, a_h=a, b_h=b$ can be computed by $p^{\mu, \nu}_h (s, a, b) = p_{1:h} (s) \mu_{1:h} (s, a) \nu_{1:h} (s, b)$, where
\begin{align*}
  p_{1:h} (s) &:= p_0 (s_1) \textstyle\prod_{h'=1}^{h-1} p_{h'} \left(s_{h'+1} | s_{h'}, a_{h'}, b_{h'}\right),
  \\
  \mu_{1:h} (s, a) &:= \mu_{1:h} (x(s), a) := \textstyle\prod_{h'=1}^{h} \mu_{h'} \left(a_{h'} | x\left(s_{h'}\right)\right),
  \\
  \nu_{1:h} (s, b) &:= \nu_{1:h} (y(s), b) := \textstyle\prod_{h'=1}^{h} \nu_{h'} \left(b_{h'} | y\left(s_{h'}\right)\right).
\end{align*}
With abuse of notation, we let $\mu_{1:h-1} (s) := \mu_{1:h-1} (x(s)) := \mu_{1:h-1} (s_{h-1}, a_{h-1})$, $p^{\mu, \nu}_h (s) := p_{1:h} (s) \mu_{1:h-1} (s, a) \nu_{1:h-1} (s, b)$ and $p^{\mu, \nu}_h (x) := \sum_{s \in x(s)} p^{\mu, \nu}_h (s)$ for any information set $x \in \cX_h$. We use $\nu_{1:h-1}$ similarly.

\paragraph{Bandit feedback}

We assume that the value of $r_h (s, a, b)$ is revealed to the players only when actions $a \in \cA$ and $b \in \cB$ are taken in a state $s \in \cS$ at time step $h$. Notice that the players are not aware of the underlying state. Furthermore, we assume that the players know neither the state transition dynamics nor the set of states $\cS$. Such limitations impose a significant difficulty as the players need to carefully play the game trying different actions to gain the information of the game.

\begin{remark}
\label{rem:difference_with_jin}
  \citet{jin2020learning} consider a similar setting (from the view point of the max-player) of learning adversarial MDPs with bandit feedback wherein the reward function is chosen by an adversary. Our setting is different in that the players have only \emph{imperfect information}, and that the \emph{state transition dynamics is changing due to the learning opponents}. Nonetheless, the tree structure and perfect-recall assumptions allow a simple and efficient model-free algorithm that we provide.
\end{remark}

\paragraph{Regret and Nash Equilibrium (NE)}

For policies $\mu$ and $\nu$ we define the expected return (of the max-player) by $V^{\mu, \nu} := \E^{\mu, \nu} [ \sum_{h=1}^H r_h (s_h, a_h, b_h) ]$\,, where $\E^{\mu, \nu}$ means that actions are selected as described above. For sequences of policies $\left(\mu^t\right)_{t \in [T]} \in \maxpi^T$ and $\left(\nu^t\right)_{t \in [T]} \in \minpi^T$, the regret of the max-player, relative to some policy $\mu^\dagger \in \maxpi,$ is defined as
\begin{align}
  \regret^T_{\mathrm{max}} (\mu^\dagger) := \sum_{t=1}^T \pp{V^{\mu^\dagger, \nu^t} - V^{\mu^t, \nu^t}}\,.\label{eq:max regret}
\end{align}
Similarly, $\sum_{t=1}^T (V^{\mu^t, \nu^t} - V^{\mu^t, \nu^\dagger})$ is the min-player's regret relative to some $\nu^\dagger \in \minpi$.

Our aim is to compute a NE. The following well-known folklore theorem,\footnote{For example see \citet{farina2019optimistic} or \citet{lanctot2009monte}.} which we prove in Appendix~\ref{sec:proof of folklore theorem}, states that this problem can be converted to the regret minimization problem.
\begin{theorem}\label{theorem:nash and regret}
  For each $(x_h, a_h, h) \in \cX_h \times \cA \times [H]$, define the average profile $(\overline{\mu}, \overline{\nu})$ by
  \begin{align}
      \overline{\mu}_h (a_h | x_h) := \frac{\sum_{t=1}^T \mu^t_{1:h} (x_h, a_h)}{\sum_{t=1}^T \mu^t_{1:h-1} (x_h)}
      \text{\quad and \quad}
      \overline{\nu}_h (b_h | y_h) := \frac{\sum_{t=1}^T \nu^t_{1:h} (y_h, b_h)}{\sum_{t=1}^T \nu^t_{1:h-1} (y_h)}\,\CommaBin\label{eq:mixture policy}
  \end{align}
  if the sum of the denominator is non-zero, otherwise as the uniform distribution over actions. If for some non-negative real value $\varepsilon$, we have that $(\regret^T_{\mathrm{max}} (\mu^\dagger) + \regret^T_{\mathrm{min}} (\nu^\dagger)) / T \leq \varepsilon$ for any profile $(\mu^\dagger, \nu^\dagger)$, then $(\overline{\mu}, \overline{\nu})$ are $\varepsilon$-NE, i.e., $\max_{\mu \in \maxpi} V^{\mu, \overline{\nu}} - \min_{\nu \in \minpi} V^{\overline{\mu}, \nu}  \leq \varepsilon$.
\end{theorem}

Given Theorem~\ref{theorem:nash and regret}, we consider how to minimize the regret for the max- and min-player; or how to control the regret such that it grows sublinearly. The subsequent section presents an algorithm, which we call implicit exploration online mirror descent (\ouralgo), that accomplishes this goal.
\section{Implicit Exploration Online Mirror Descent (\texorpdfstring{\ouralgo}{IXOMD})}

\begin{algorithm}[t!]
  \SetAlgoLined
  \DontPrintSemicolon
  \KwIn{IX hyper-parameter $\gamma \in (0, \infty)$ and \OMD's learning rate $\eta \in (0, \infty)$.}
  \KwOut{A near-NE policy for the max-player.}
  Initialize $\mu_h^1 (a_h|x_h) \gets 1  / A$  for each $(x_h, a_h, h) \in \cX_h \times \cA \times [H]$.\;
  \For{$t=1, \ldots, T$}{
    \For{$h = 1, \ldots, H$}{
      Observe $x_h^t$, execute $a_h^t \sim \mu_h^t (\cdot|x_h^t)$, and receive $r_h^t$.\;
    }
    Set $Z^t_{H+1} \gets 1$.\;
    \For{$h = H, \ldots, 1$}{
      Construct the IX loss estimate $\widetilde{\ell}_h^{\,t}$ by
      \begin{equation*}
        \widetilde{\ell}_h^{\,t} \gets \frac{1 - r_h^t}{\mu_{1:h}^t (x_h^t, a_h^t) + \gamma}\,\cdot
      \end{equation*}\\
      Compute for each $h \in [H]$ with $Z^t_{H+1} \gets 1$
      \begin{equation*}
        Z^t_h \gets 1 - \mu^t_h (a_h^t | x_h^t) + \mu^t_h (a_h^t | x_h^t) \exp \pp{ - \eta \widetilde{\ell}^{\,t}_h + \log Z_{h+1}^t }\,.
      \end{equation*}\\
      Update $\mu^t$ to $\mu^{t+1}$ at $x_h^t$ by
      \begin{equation*}
        \mu_h^{t+1} (a_h | x_h^t) \gets 
        \begin{cases}
          \mu^t_h (a_h | x_h^t) \exp \pp{ - \eta \widetilde{\ell}^{\,t}_h + \log Z_{h+1}^t - \log Z^t_h } & \text{if $a_h = a_h^t$}
          \\
          \mu^t_h (a_h | x_h^t) \exp \pp{ - \log Z^t_h } & \text{otherwise}
        \end{cases}
      \end{equation*}\\
      and $\mu^{t+1} (\cdot | x_h) \gets \mu^t (\cdot | x_h)$ at other information sets $x_h \in \cX_h$\,.\;
    }
  }
  \Return{Policy $\overline{\mu}$ which is the average of $\mu_1,\ldots,\mu_T$ defined in Theorem~\ref{theorem:nash and regret}.}
  \caption{\ouralgo for the Max-Player}\label{alg:ouralgo}
\end{algorithm}

Due to the symmetry of the players, it suffices to consider only the learning of the max-player. Therefore we mainly focus on it and denote the max-player's regret \eqref{eq:max regret} by $\regret^T (\mu^\dagger)$. We first convert the original regret minimization problem into a adversarial linear bandits one. Then we give an explanation behind the use of implicit exploration and introduce our algorithm, \ouralgo, whose pseudocode is given in Algorithm~\ref{alg:ouralgo}. For simplicity, we give a simple-to-read but inefficient version. In Appendix~\ref{sec:memory efficient implementation}, we provide a practical version, whose computational and memory complexity are detailed in Section~\ref{sec:implementation}.

\paragraph{Additional notation} For a policy $\mu \in \maxpi$ and a sequence of functions $f := (f_h)_{h \in [H]},$ where $f_h: \cX_h \times \cA \rightarrow \R$, we denote the scalar product $\sum_{h \in [H]} \sum_{x_h \in \cX_h, a \in \cA} \mu_{1:h} (x_h, a_h) f_h (x_h, a_h)$ by $\angl{\mu, f}$. We let $\cF^{t-1}$ be the $\sigma$-algebra generated by variables up to the beginning of the $t$-th episode, i.e., $\{s_h^\tau, a_h^\tau, b_h^\tau\}_{h \in [H], \tau \in [t-1]}$. We let $\E^{t-1} [\cdot] := \E [\cdot | \cF^{t-1}]$.

\subsection{Conversion to online linear regret minimization}
\label{sec:conversion_to_online}

Note that for any profile $(\mu, \nu)$, we have
\begin{align*}
  V^{\mu, \nu}
  &= \sum_{h = 1}^H \sum_{s_h \in \cS_h, a_h \in \cA, b_h \in \cB} p_{1:h} (s_h) \mu_{1:h} (s_h, a_h) \nu_{1:h} (s_h, b_h) r_h (s_h, a_h, b_h)
  \\
  &= \sum_{h = 1}^H \sum_{x_h \in \cX_h, a_h \in \cA} \mu_{1:h} (x_h, a_h) \sum_{s_h \in x_h, b_h \in \cB} p_{1:h} (s_h) \nu_{1:h} (s_h, b_h) r_h (s_h, a_h, b_h),
\end{align*}
where we used the facts that $\mu_{1:h}$ is dependent on $(x_h, a_h)$ rather than $(s_h, a_h)$, and $\sum_{s_h \in \cS_h} f(s_h) = \sum_{x_h \in \cX_h} \sum_{s_h \in x_h} f(s_h)$ for any function $f: \cS \rightarrow \R$. Therefore defining a loss by
\begin{align*}
  \ell_h^{\,t} (x_h, a_h) := \sum_{s_h \in x_h, b_h \in \cB} p_{1:h} (s_h) \nu_{1:h}^t (s_h, b_h) \pp{1 - r_h (s_h, a_h, b_h)}\,,
\end{align*}
we can rewrite the regret \eqref{eq:max regret} as\footnote{As introduced at \textbf{Additional notation}, $\langle \mu^t, \widetilde{\ell}^{\,t} \rangle = \sum_{h = 1}^H \sum_{x_h \in \cX_h, a \in \cA} \mu^t_{1:h} (x_h, a_h) \widetilde{\ell}_h^{\,t}(x_h, a_h)$. Hence the meaning of $\mu^t$ here is abused, and we are viewing it as a sequence $(\mu^t_{1:h})_{h \in [H]}$ of functions. In this case, $\mu^t$ must satisfy the following two conditions: (non-negativity) $\mu^t_{1:h} (x_h, a_h) \geq 0$ for any $x_h \in \cX_h$ and $h \in [H]$; (consistency) $\sum_{a_h \in \cA} \mu^t_{1:h} (x_h, a_h) = \mu^t_{1:h-1} (x_{h-1}, a_{h-1})$ for any $x_h \in \cX_h$ and $h \in \{2, \ldots, H\}$, where $(x_{h-1}, a_{h-1})$ is a unique predecessor of $x_h$, and $\sum_{a_1 \in \cA} \mu^t_{1:1} (x_1, a_1) = 1$ for any $x_1 \in \cX_1$. Nonetheless there is a bijective mapping between $\maxpi$ and the set of $\mu^t$ satisfying these two conditions. Therefore we do not discern these two sets.}
\begin{align}
  \regret^T (\mu^\dagger)
  = \sum_{t = 1}^T \angl{\mu^t - \mu^\dagger, \ell^t}\,.\label{eq:online linear regret}
\end{align}
This result tells us that we may convert the original regret minimization problem to a linear one in which we choose $\mu^t$ such that $\regret^T (\mu^\dagger)$ grows sublinearly.

\subsection{Loss estimation and implicit exploration}

To solve the regret minimization problem \eqref{eq:online linear regret} with bandit feedback, we need to estimate $\ell^t$. An unbiased importance sampling estimator is 
\begin{equation}\label{eq:IS.estimate}
\widehat{\ell}_h^{\,t} (x_h, a_h) := \frac{\I_{\{x_h=x_h^t, a_h=a_h^t\}}}{\mu_{1:h}^t (x_h, a_h)} \pp{1 - r_h^t}.
\end{equation}
However, instead, we estimate the loss by
\begin{align}\label{eq:IX loss estimate}
  \widetilde{\ell}_h^{\,t} (x_h, a_h) := \frac{\I_{\{x_h=x_h^t, a_h=a_h^t\}}}{\mu_{1:h}^t (x_h, a_h) + \gamma} \pp{1 - r_h^t}\,,
\end{align}
where $\gamma$ is a positive real value and a hyper-parameter. This estimator is used by implicit exploration in bandits (IX, \citealp{kocak2014efficient,neu2015explore}; \citealp[Chapter 12]{lattimore2020bandit_book}),
and we therefore refer to it as the IX estimator. Note that IX uses a biased estimate, but it prevents the variance of the IX estimator from becoming too large.

\subsection{Efficient implementation, Space- and Time-Complexities}\label{sec:implementation}

Given a loss estimate, we find $\mu^{t+1}$ by solving
\begin{align}\label{eq:OMD update}
  \mu^{t+1} := \argmin_{\mu \in \maxpi} \eta \angl{\mu, \widetilde{\ell}^{\,t}} + \D \pp{ \mu \mbb \mu^t}\,,
\end{align}
where $\D$ is the \textit{dilated} entropy distance-generating function \citep{kroer2015faster} defined by
\begin{align*}
  \D \pp{ \mu \mbb \mu' } := \sum_{h=1}^H \sum_{x_h \in \cX_h, a_h \in \cA} \mu_{1:h} (x_h, a_h) \log \frac{\mu_h (a_h | x_h)}{\mu_h' (a_h | x_h)}\,\cdot
\end{align*}
Note that $\D$ is a Bregman divergence. The update in~\eqref{eq:OMD update} has an easy implementation, as explained next. For more details of its derivation, please refer to Appendix~\ref{sec:details of implementation}. To compute a new policy, we first need to compute for each $h \in [H]$, 
\begin{align*}
  Z^t_h
  &:= \sum_{a_h \in \cA} \mu^t_h (a_h | x_h^t) \exp \pp{ \I_{\{a_h=a_h^t\}} \pp{ - \eta \widetilde{\ell}^{\,t}_H (x_h^t, a_h) + \log Z_{h+1}^t} }
  \\
  &= 1 - \mu^t_h (a_h^t | x_h^t) + \mu^t_h (\{a_h^t | x_h^t\}) \exp \pp{ - \eta \widetilde{\ell}^{\,t}_H (x_h^t, a_h) + \log Z_{h+1}^t },
\end{align*}
with $Z^t_{H+1} := 1$. Then, we can compute a new policy by
\begin{align}
  \mu_h^{t+1} (a_h | x_h^t) = \mu^t_h (a_h | x_h^t) \exp \pp{ \I_{\{a_h=a_h^t\}} \pp{ - \eta \widetilde{\ell}^{\,t}_h (x_h^t, a_h) + \log Z_{h+1}^t} - \log Z^t_h }\,.\label{eq:policy update}
\end{align}
Note that this policy is updated \textit{only} at the information sets visited along the $t$-th trajectory. This implies that the update requires $\cO(HA)$ time-complexity per episode. Therefore the learning of the policies require $\cO(THA)$ time-complexity in total.

Interestingly, the update of the average policy $\overline{\mu}$ can also be performed in a semi-online way, see Appendix~\ref{sec:efficient computation of average policy}. This method has a total time-complexity of $\cO(THA + \min(TH, X) A)$ and space-complexity of $\cO(\min(TH, X) A)$. Please refer to Algorithm~\ref{alg:memory efficient ouralgo} in Appendix~\ref{sec:memory efficient implementation} for a pseudocode of this practical implementation.

Algorithm~\ref{alg:memory efficient ouralgo} requires a post-hoc computation that is the source of $\cO(\min(TH, X) A)$ time-complexity. It is possible to defer the post-hoc computation until $\overline{\mu} (\cdot | x_h)$ is needed for playing a game. In this case, the computation of $\overline{\mu} (\cdot | x_h)$ is performed while traversing a game tree. For one traversal, $\overline{\mu} (\cdot | x_h)$ is computed for each $h$, and the total time-complexity is $\cO(HA)$. The space-complexity is unchanged and is $\cO(\min(TH, X)A)$.

\section{Theoretical Analysis of \texorpdfstring{\ouralgo}{IXOMD}}\label{sec:analysis}

We now analyze \ouralgo. It has the following guarantee, which we shall prove in the present section.
\begin{theorem}[regret bound of \ouralgo]\label{theorem:regret of IXOMD}
  Let $\delta \in (0, 1)$. The regret \eqref{eq:max regret} satisfies the following guarantee with probability at least $1-\delta$
  \begin{align*}
    \max_{\mu^\dagger \in \maxpi} \regret^T (\mu^\dagger)
    &\leq H \sqrt{2T \iota} + \gamma TXA + \frac{X \iota}{2\gamma} + \frac{X \log A}{\eta} + \eta (1 + H) TXA + \frac{\eta (1 + H) H \iota}{2\gamma}\,\CommaBin
  \end{align*}
  where $\iota := \log (3HXA/\delta)$. In particular $\eta = \sqrt{\dfrac{\log A}{T(1+H)A}}$ and $\gamma = \sqrt{\dfrac{\iota}{2TA}}$ result in
  \begin{align*}
    \max_{\mu^\dagger \in \maxpi} \regret^T (\mu^\dagger)
    \leq H \sqrt{2T \iota} + X \sqrt{2TA\iota} + X \sqrt{T (1+H) A \log A} + H \sqrt{ \dfrac{(1 + H) \iota \log A}{2} }\,.
  \end{align*}
\end{theorem}

\begin{remark}\label{rem:adversarial_mdp}
  We emphasize that this result is agnostic of the min-player. In particular, the same result holds for learning in a partially observable MDP with adversarial state-transition dynamics and reward function, as long as assumptions similar to the tree-like structure and perfect recall hold.
\end{remark}

\begin{remark} \label{rem:tuning_lr}
  In Theorem~\ref{theorem:regret of IXOMD}, we adjusted $\eta$ and $\gamma$ using $T$, $H$, $X$, and $A$. Even when we know~$T$ only, setting $\eta = 1/\sqrt{T}$ and $\gamma = 1/\sqrt{T}$ guarantees an upper-bound of the order of $\tcO(XA\sqrt{T})$\footnote{ We recall that we hide with $\tcO$ poly-log terms in $e^H,T,X,A,1/\delta$.}. If we additionally know $H$ and~$A$ (which is likely to be the case), but {\em do not} know $X$, setting $\eta = \sqrt{\log A/(T(1+H)A)}$ and $\gamma = 1/\sqrt{2TA}$ still results in an upper-bound of the order of $\tcO ( X \sqrt{TA} )$.
\end{remark}

A similar result holds for the min-player thanks to the symmetry. From Theorem~\ref{theorem:nash and regret} and \ref{theorem:regret of IXOMD}, it follows that the average profile $(\overline{\mu}, \overline{\nu})$ is close to a Nash equilibrium with high probability.

\begin{corollary}
  Suppose that both max- and min-players learn their policies by \ouralgo with the setting\footnote{Note that $A$ and $X$ must be replaced with $B$ and $Y$ (resp.) for the min-player's $\eta$ and $\gamma$. Also note that to archive the same order of a bound as the one shown in this corollary, we need neither $X$ nor $Y$.} of $\eta$ and $\gamma$ in Theorem~\ref{theorem:regret of IXOMD}. Then with probability at least $1 -  \delta$, the average profile $(\overline{\mu}, \overline{\nu})$ defined in Theorem~\ref{theorem:nash and regret} is $\varepsilon$-Nash equilibrium, with
  \begin{align*}
    \varepsilon := \tcO\left(\frac{1}{\sqrt{T}}\left( X\sqrt{A} + Y\sqrt{B}
    \right)\right).
  \end{align*}
\end{corollary}

\subsection{Proof of Theorem~\ref{theorem:regret of IXOMD}}

Now we start the proof of Theorem~\ref{theorem:regret of IXOMD}. In the first step, we decompose the regret \eqref{eq:online linear regret} to three terms:
\begin{align}
  \regret^T (\mu^\dagger)
  = \underbrace{\sum_{t = 1}^T \angl{\mu^t, \ell^{\,t} - \widetilde{\ell}^{\,t}}}_{\text{BIAS 1}} - \underbrace{\sum_{t = 1}^T \angl{\mu^\dagger, \ell^{\,t} - \widetilde{\ell}^{\,t}}}_{\text{BIAS 2}} + \underbrace{\sum_{t = 1}^T \angl{\mu^t - \mu^\dagger, \widetilde{\ell}^{\,t}}}_{\text{REGRET}}\,.\label{eq:regret decomposition}
\end{align}
Then, we prove a high-probability upper-bound for each term. After deriving each upper-bound, Theorem~\ref{theorem:regret of IXOMD} follows simply by taking the union bound
over the three terms.

For proving the upper-bounds, we need the following lemma, which almost immediately follows from Lemma~1 of \citet{neu2015explore} (also see Lemma~12.2 of \citet{lattimore2020bandit_book} for a more general statement). For completeness we prove it below with our notation.

\begin{lemma}\label{lemma:generalized Neu lemma}
  Let $\delta \in (0, 1)$ and $\gamma \in (0, \infty)$. Fix $h \in [H]$, and let $\alpha^t (x_h, a_h) \in [0, 2\gamma]$ be $\cF^{t-1}$-measurable random variable for each $(x_h, a_h) \in \cX_h \times \cA$. Then with probability at least $1-\delta$
  \begin{align*}
    \sum_{t=1}^T \sum_{x_h \in \cX_h, a_h \in \cA} \alpha^t (x_h, a_h) \pp{ \widetilde{\ell}^{\,t}_h (x_h, a_h) - \ell^{\,t}_h (x_h, a_h) } \leq \log \frac{1}{\delta}
  \end{align*}
\end{lemma}

\begin{proof}
Let $\widehat{\ell}^{\,t}_h (x_h, a_h)$ be the unbiased importance-sampling estimate of $\ell^{\,t}_h (x_h, a_h)$ defined in Equation~\ref{eq:IS.estimate}. Then for any $t \in [T]$, $x_h \in \cX_h$, and $a_h \in \cA$,
\begin{align*}
  \widetilde{\ell}^{\,t}_h (x_h, a_h)
  &= \frac{1 - r_h^t}{\mu_{1:h}^t (x_h, a_h) + \gamma} \I_{\{x_h=x_h^t, a_h=a_h^t\}}
  \\
  &\leq \frac{1 - r_h^t}{\mu_{1:h}^t (x_h, a_h) + \gamma \pp{1 - r_h^t}} \I_{\{x_h=x_h^t, a_h=a_h^t\}}
  \\
  &= \frac{1}{\beta} \frac{\beta \pp{1 - r_h^t} \I_{\{x_h=x_h^t, a_h=a_h^t\}} / \mu_{1:h}^t (x_h, a_h)}{1 + \gamma \pp{1 - r_h^t} \I_{\{x_h=x_h^t, a_h=a_h^t\}} / \mu_{1:h}^t (x_h, a_h)}
  = \frac{1}{\beta} \frac{\beta \widehat{\ell}^{\,t}_h (x_h, a_h)}{1 + \beta \widehat{\ell}^{\,t}_h (x_h, a_h) / 2}
  \\
  &\leq \frac{1}{\beta} \log \pp{1 + \beta \widehat{\ell}^{\,t}_h (x_h, a_h)} \CommaBin
\end{align*}
where $\beta := 2\gamma$, and the last inequality follows from $\dfrac{z}{1+z/2} \leq \log (1+z)$ for any $z \in [0, \infty)$.

Let $\widetilde{\lambda}^{\,t} := \sum_{x_h \in \cX_h, a_h \in \cA} \alpha^t (x_h, a_h) \widetilde{\ell}^{\,t}_h (x_h, a_h)$ and $\lambda^{\,t} := \sum_{x_h \in \cX_h, a_h \in \cA} \alpha^t (x_h, a_h) \ell^{\,t}_h (x_h, a_h)$. Note that we want to show $\sum_{t=1}^T (\widetilde{\lambda}^{\,t} - \lambda^{\,t}) \leq \log (1/\delta)$. Using the above inequality, we deduce that
\begin{align*}
  \E^{t-1} \bb{ \exp \pp{ \widetilde{\lambda}^{\,t} } }
  &\leq \E^{t-1} \bb{ \exp \pp{ \sum_{x_h \in \cX_h, a_h \in \cA} \frac{\alpha^t (x_h, a_h)}{\beta} \log \pp{1 + \beta \widehat{\ell}^{\,t}_h (x_h, a_h)} } }
  \\
  &\leq \E^{t-1} \bb{ \prod_{x_h \in \cX_h, a_h \in \cA} \pp{1 + \alpha^t (x_h, a_h) \widehat{\ell}^{\,t}_h (x_h, a_h)} }
  \\
  &\leq \E^{t-1} \bb{ 1 + \sum_{x_h \in \cX_h, a_h \in \cA} \alpha^t (x_h, a_h) \widehat{\ell}^{\,t}_h (x_h, a_h) }
  \\
  &= 1 + \sum_{x_h \in \cX_h, a_h \in \cA} \alpha^t (x_h, a_h) \ell^{\,t}_h (x_h, a_h)
  \\
  &\leq \exp \pp{ \sum_{x_h \in \cX_h, a_h \in \cA} \alpha^t (x_h, a_h) \ell^{\,t}_h (x_h, a_h) }
  = \exp \pp{ \lambda^{\,t} }\CommaBin
\end{align*}
where the second line follows from $z \log (1+z') \leq \log (1+zz')$ for any $z \in [0, 1]$ and $z' \in (-1, \infty)$, the third line follows from $\widehat{\ell}^{\,t}_h (x_h, a_h) \widehat{\ell}^{\,t}_h (x_h', a_h') = 0$ for any $(x_h, a_h) \neq (x_h', a_h')$, and the last line follows from $1 + z \leq \exp(z)$ for any $z \in \R$.

Define $Z_t := \exp ( \widetilde{\lambda}^{\,t} - \lambda^{\,t} )$ and $M_t := \prod_{u=1}^t Z_u$. From the above inequality, we have that $\E [M_t] = \E \bb{\E^{t-1} \bb{M_t}} = \E \bb{M_{t-1} \E^{t-1} \bb{Z_t}} \leq \E \bb{M_{t-1}} \leq \cdots \leq 1$. As a result, Markov's inequality implies
\begin{align*}
  \Pr \pp{\sum_{t=1}^T (\widetilde{\lambda}^{\,t} - \lambda^{\,t}) \geq \log \frac 1\delta} = \Pr \pp{\log M_T \geq \log \frac 1\delta} = \Pr (M_t \delta \geq 1) \leq \E[M_t] \delta \leq \delta\,.
\end{align*}
This concludes the proof.
\end{proof}

We first prove an upper-bound of BIAS 1 shown below.
\begin{lemma}[upper-bound of BIAS 1]
  Let $\delta \in (0, 1)$. For any $\mu^\dagger \in \maxpi$ it holds with probability at least $1-\delta/3$ that $\text{BIAS 1} \leq H \sqrt{2T \iota} + \gamma TXA$.
\end{lemma}

\begin{proof}
To see that this is true, we first deduce that
\begin{align*}
  \angl{\mu^t, \widetilde{\ell}^{\,t}}
  &= \sum_{h=1}^H \sum_{x_h \in \cX_h, a_h \in \cA} \mu^t_{1:h} (x_h, a_h) \frac{\I_{\{x_h=x_h^t, a_h=a_h^t\}}}{\mu_{1:h}^t (x_h, a_h) + \gamma} \pp{1 - r_h^t}
  \\
  &\leq \sum_{h=1}^H \sum_{x_h \in \cX_h, a_h \in \cA} \I_{\{x_h=x_h^t, a_h=a_h^t\}}
  = \sum_{h=1}^H 1
  = H\,,
\end{align*}
where the inequality follows from $\mu^t_{1:h} (x_h, a_h) / \pp{\mu_{1:h}^t (x_h, a_h) + \gamma} \leq 1$, and $0 \leq 1 - r_h^t \leq 1$. By Hoeffding-Azuma inequality, we deduce that $\sum_{t = 1}^T \langle{\mu^t, \widetilde{\ell}^{\,t} - \E^{t-1}[ \widetilde{\ell}^{\,t}] \rangle} \geq - H \sqrt{2T \log (3/\delta)} \geq - H \sqrt{2T \iota}$ with probability at least $1 - \delta/3$. (The final inequality is to simplify the result.) Next, we deduce that
\begin{align*}
  \angl{\mu^t, \ell^{\,t} - \E^{t-1}\bb{\widetilde{\ell}^{\,t}}}
  &= \sum_{h=1}^H \sum_{x_h \in \cX_h, a_h \in \cA} \mu^t_{1:h} (x_h, a_h) \pp{1 - \frac{\mu_{1:h}^t (x_h, a_h)}{\mu_{1:h}^t (x_h, a_h) + \gamma} } \ell_h^{\,t} (x_h, a_h)
  \\
  &= \sum_{h=1}^H \sum_{x_h \in \cX_h, a_h \in \cA} \mu^t_{1:h} (x_h, a_h) \frac{\gamma \ell_h^{\,t} (x_h, a_h)}{\mu_{1:h}^t (x_h, a_h) + \gamma}
  \\
  &\leq \gamma \sum_{h=1}^H \sum_{x_h \in \cX_h, a_h \in \cA} \ell_h^{\,t} (x_h, a_h)
  \leq \gamma \sum_{h=1}^H \sum_{x_h \in \cX_h, a_h \in \cA} 1
  \leq \gamma XA\,,
\end{align*}
where the first inequality follows from $\mu^t (x_h, a_h) / \pp{\mu_{1:h}^t (x_h, a_h) + \gamma} \leq 1$, and the last inequality follows from $\sum_{h=1}^H \sum_{x_h \in \cX_h, a_h \in \cA} 1 = \sum_{h=1}^H \abs{\cX_h} A = XA$. Combining both bounds, we obtain the claimed result.
\end{proof}

Next we prove an upper-bound of BIAS 2.
\begin{lemma}[upper-bound of BIAS 2]
  Let $\delta \in (0, 1)$. For any $\mu^\dagger \in \maxpi$ it holds with probability at least $1-\delta/3$ that $\text{BIAS 2} \leq X \iota/ (2\gamma)$.
\end{lemma}

\begin{proof}
Note that
\begin{align*}
  &\sum_{t=1}^T \sum_{h=1}^H\sum_{x_h \in \cX_h, a \in \cA} \mu^\dagger_{1:h} (x_h, a_h) \pp{ \widetilde{\ell}^{\,t}_h (x_h, a_h) - \ell^{\,t}_h (x_h, a_h) }
  \\
  &= \sum_{h=1}^H \sum_{x_h \in \cX_h, a_h \in \cA} \mu^\dagger_{1:h} (x_h, a_h) \underbrace{\sum_{t=1}^T \sum_{x_h' \in \cX_h, a_h' \in \cA} \I_{\{x_h'=x_h, a_h'=a_h\}} \pp{ \widetilde{\ell}^{\,t}_h (x_h', a_h') - \ell^{\,t}_h (x_h', a_h') }}_{\clubsuit}\,.
\end{align*}
Now we can apply Lemma~\ref{lemma:generalized Neu lemma} to $\clubsuit$ (with $\alpha^t (x_h', a_h') = 2\gamma \I_{\{x_h'=x_h, a_h'=a_h\}}$) and deduce that
\begin{align*}
    \sum_{t=1}^T \sum_{x_h' \in \cX_h, a_h' \in \cA} \I_{\{x_h'=x_h, a_h'=a_h\}} \pp{ \widetilde{\ell}^{\,t}_h (x_h', a_h') - \ell^{\,t}_h (x_h', a_h') }
    \leq \frac{\iota}{2\gamma}
\end{align*}
with probability at least $\delta / (3HXA)$. For every $(h, x_h, a_h)$ we have the same result. By the union bound, we obtain that
\begin{align*}
  &\sum_{t=1}^T \sum_{h=1}^H \sum_{x_h \in \cX_h, a \in \cA} \mu^\dagger_{1:h} (x_h, a_h) \pp{ \widetilde{\ell}^{\,t}_h (x_h, a_h) - \ell^{\,t}_h (x_h, a_h) }
  \\
  &\leq \frac{\iota}{2\gamma} \sum_{h=1}^H \sum_{x_h \in \cX_h, a_h \in \cA} \mu^\dagger_{1:h} (x_h, a_h)
  \leq \frac{X \iota}{2\gamma}\,.
\end{align*}
with probability at least $\delta / 3$.  
\end{proof}

Finally we prove the following upper-bound of REGRET.

\begin{lemma}[upper-bound of REGRET]\label{lemma:regret term upper bound}
  Let $\delta \in (0, 1)$. For any $\mu^\dagger \in \maxpi$ it holds with probability at least $1-\delta/3$ that
  \begin{align*}
      \text{REGRET} \leq \frac{X \log A}{\eta} + \eta (1 + H) TXA + \frac{\eta (1 + H) H \iota}{2\gamma}\,.
  \end{align*}
\end{lemma}

To prove the upper-bound, we first connect $\langle \mu^t - \mu^\dagger, \widetilde{\ell}_h^t \rangle$ to divergences between $\mu^\dagger$, $\mu^t$ and $\mu^{t+1}$. To this end the following technical lemma turns out to be useful.

\begin{lemma}\label{lemma:kl diff and return}
  For any policy $\mu \in \maxpi$ we have that
  \begin{align*}
    \D \pp{\mu \| \mu^{t+1}} - \D \pp{\mu \| \mu^t} = \eta\angl{\mu, \widetilde{\ell}^{\,t}} + \log Z^t_1
  \end{align*}
\end{lemma}

\begin{proof}
  From the form of policy updates \eqref{eq:policy update} we may deduce that
  \begin{align*}
    &\D \pp{\mu \| \mu^{t+1}} - \D \pp{\mu \| \mu^t}
    \\
    &\hspace{3em}= \sum_{h=1}^H \sum_{x_h \in \cX_h, a_h \in \cA} \mu_{1:h} (x_h, a_h) \log \frac{\mu_h^{t} (a_h | x_h)}{\mu_h^{t+1} (a_h | x_h)}
    \\
    &\hspace{3em}= \sum_{h=1}^H \mu_{1:h} (x_h^t, a_h^t) \pp{ \eta \widetilde{\ell}^{\,t}_h (x_h^t, a_h^t) - \log Z_{h+1}^t } + \sum_{h=1}^H \mu_{1:h-1} (x_h^t) \log Z^t_h\,.
  \end{align*}
  By noting that
  \begin{align*}
    &- \sum_{h=1}^H \mu_{1:h} (x_h^t, a_h^t) \log Z_{h+1}^t + \sum_{h=1}^H \mu_{1:h-1} (x_h^t) \log Z_h^t
    \\
    &\hspace{3em}= - \sum_{h=1}^{H-1} \mu_{1:h} (x_{h+1}^t) \log Z_{h+1}^t + \sum_{h=1}^H \mu_{1:h-1} (x_h^t) \log Z_h^t
    \\
    &\hspace{3em}= - \sum_{h=2}^{H} \mu_{1:h-1} (x_h^t) \log Z_h^t + \sum_{h=1}^H \mu_{1:h-1} (x_h^t) \log Z_h^t
    = \log Z_1^t\,,
  \end{align*}
  we deduce the claimed result.
\end{proof}

Now we are ready to prove Lemma~\ref{lemma:regret term upper bound}.

\begin{proof}[proof of Lemma~\ref{lemma:regret term upper bound}]

From a fact that
\begin{align*}
  &\D \pp{\mu^\dagger \| \mu^t} - \D \pp{\mu^\dagger \| \mu^{t+1}} + \D \pp{\mu^t \| \mu^{t+1}}
  \\
  &\hspace{3em}= - \pp{\D \pp{\mu^\dagger \| \mu^{t+1}} - \D \pp{\mu^\dagger \| \mu^t}} + \D \pp{\mu^t \| \mu^{t+1}} - \D \pp{\mu^t \| \mu^t}\,,
\end{align*}
and Lemma~\ref{lemma:kl diff and return}, we have that $\eta \langle \mu^t - \mu^\dagger, \widetilde{\ell}^{\,t} \rangle = \D \pp{\mu^\dagger \| \mu^t} - \D \pp{\mu^\dagger \| \mu^{t+1}} + \D \pp{\mu^t \| \mu^{t+1}}$. Taking the sum over $t$ noting that $\D \pp{\mu^\dagger \| \mu^{T+1}} \geq 0$, we deduce that
\begin{align*}
  \eta \sum_{t=1}^T \angl{\mu^t - \mu^\dagger, \widetilde{\ell}^{\,t}}
  \leq \D \pp{\mu^\dagger \| \mu^1} + \sum_{t=1}^T \D \pp{\mu^t \| \mu^{t+1}}\,.
\end{align*}
We need to upper-bound the two terms on the right side.

The first term is easy to upper-bound. From the definition of the divergence and the choice for the first policy we have
\begin{align*}
  \D \pp{ \mu^\dagger \mbb \mu^1}
  &= \sum_{h=1}^H \sum_{x_h \in \cX_h, a_h \in \cA} \mu^\dagger_{1:h} (x_h, a_h) \log \frac{\mu^\dagger_h (a_h | x_h)}{\mu^1_h (a_h | x_h)}
  \\
  &\leq - \sum_{h=1}^H \sum_{x_h \in \cX_h, a_h \in \cA} \mu^\dagger_{1:h} (x_h, a_h) \log \mu^1_h (x_h, a_h)
  \\
  &= \log A \sum_{h=1}^H \sum_{x_h \in \cX_h, a_h \in \cA} \mu^\dagger_{1:h} (x_h, a_h)
  \leq X \log A\,.
\end{align*}

In contrast bounding the second term is somewhat lengthy and technical. For brevity we use the following notations: $\widetilde{\ell}^{\,t}_h := \widetilde{\ell}^{\,t}_h (x_h^t, a_h^t)$, $\mu_h^t := \mu_h^t (x_h^t, a_h^t)$ and $\mu_{h:h'}^t := \mu_{h'}^t / \mu_{h}^t$, where $h' > h$.

From Lemma~\ref{lemma:kl diff and return} we have that
\begin{align*}
  \D \pp{\mu^t \| \mu^{t+1}}
  = \D \pp{\mu^t \| \mu^{t+1}} - \D \pp{\mu^t \| \mu^t}
  = \eta \angl{\mu^t, \widetilde{\ell}^{\,t}} + \log Z_1^t\,.
\end{align*}
We show that $\log Z_1^t \approx - \eta \langle \mu^t, \widetilde{\ell}^{\,t} \rangle$. Firstly we prove by induction on $h$ that for any $h$
\begin{align}
  Z_h^t = 1 + \sum_{h'=h}^H \mu_{h:h'}^t \exp \pp{-\eta \sum_{h''=h}^{h'-1} \widetilde{\ell}^{\,t}_{h''}} \pp{\exp \pp{-\eta \widetilde{\ell}^{\,t}_{h'} } - 1}\,.\label{eq:rewritten Z}
\end{align}
By definition $Z^t_H = 1 - \mu^t_H + \mu^t_H \exp \pp{ - \eta \widetilde{\ell}^{\,t}_H } = 1 + \mu^t_H \pp{ \exp \pp{ - \eta \widetilde{\ell}^{\,t}_H } - 1 }\,,$ and thus, the claim holds for $h=H$. Now suppose that the induction hypothesis is true from $H$ to $h+1$. Then
\begin{align*}
  &Z^t_h
  \\
  &= 1 - \mu^t_h + \mu^t_h \exp \pp{ - \eta \widetilde{\ell}^{\,t}_h + \log Z_{h+1}^t}
  \\
  &= 1 - \mu^t_h + \mu^t_h \exp \pp{-\eta \widetilde{\ell}^{\,t}_h} Z_{h+1}^t
  \\
  &= 1 - \mu^t_h + \mu^t_h \exp \pp{-\eta \widetilde{\ell}^{\,t}_h} \pp{ 1 + \sum_{h'=h+1}^H \mu_{h+1:h'}^t \exp \pp{-\eta \sum_{h''=h+1}^{h'-1} \widetilde{\ell}^{\,t}_{h''}} \pp{\exp \pp{-\eta \widetilde{\ell}^{\,t}_{h'} } - 1} }
  \\
  &= 1 + \mu^t_h \pp{ \exp \pp{-\eta \widetilde{\ell}^{\,t}_h} - 1 }+ \sum_{h'=h+1}^H \mu_{h:h'}^t \exp \pp{-\eta \sum_{h''=h}^{h'-1} \widetilde{\ell}^{\,t}_{h''}} \pp{\exp \pp{-\eta \widetilde{\ell}^{\,t}_{h'} } - 1}
  \\
  &= 1 + \sum_{h'=h}^H \mu_{h:h'}^t \exp \pp{-\eta \sum_{h''=h}^{h'-1} \widetilde{\ell}^{\,t}_{h''}} \pp{\exp \pp{-\eta \widetilde{\ell}^{\,t}_{h'} } - 1}\,.
\end{align*}
Therefore the equality \eqref{eq:rewritten Z} holds. Using the equality \eqref{eq:rewritten Z}, $\log (1+x) \leq x$ for any $x \in (-1, \infty)$ and $\exp(-x) \leq 1 - x + x^2$ for any $x \in (0, \infty)$, we deduce that
\begin{align*}
  \log Z_1^t
  &\leq \sum_{h=1}^H \mu_{1:h}^t \exp \pp{-\eta \sum_{h'=1}^{h-1} \widetilde{\ell}^{\,t}_{h'}} \pp{\exp \pp{-\eta \widetilde{\ell}^{\,t}_h } - 1}
  \\
  &\leq \sum_{h=1}^H \mu_{1:h}^t \exp \pp{-\eta \sum_{h'=1}^{h-1} \widetilde{\ell}^{\,t}_{h'}} \pp{-\eta \widetilde{\ell}^{\,t}_h + \eta^2 \pp{\widetilde{\ell}^{\,t}_h}^2}\,.
\end{align*}
Therefore we get
\begin{align*}
  \D \pp{\mu^t \| \mu^{t+1}}
  &\leq \eta \sum_{h=1}^H \mu_{1:h}^t \widetilde{\ell}^{\,t}_h + \sum_{h=1}^H \mu_{1:h}^t \exp \pp{-\eta \sum_{h'=1}^{h-1} \widetilde{\ell}^{\,t}_{h'}} \pp{-\eta \widetilde{\ell}^{\,t}_h + \eta^2 \pp{\widetilde{\ell}^{\,t}_h}^2}
  \\
  &= \eta \sum_{h=1}^H \mu_{1:h}^t \widetilde{\ell}^{\,t}_h \pp{1 - \exp \pp{-\eta \sum_{h'=1}^{h-1} \widetilde{\ell}^{\,t}_{h'}}} + \eta^2 \sum_{h=1}^H \mu_{1:h}^t \exp \pp{-\eta \sum_{h'=1}^{h-1} \widetilde{\ell}^{\,t}_{h'}} \pp{\widetilde{\ell}^{\,t}_h}^2\,.
\end{align*}
Using $1-\exp(-x) \leq x$ for $x \in \R$ yields
\begin{align*}
  \mu_{1:h}^t \pp{1 - \exp \pp{-\eta \sum_{h'=1}^{h-1} \widetilde{\ell}^{\,t}_{h'}}}
  \leq \eta \mu_{1:h}^t \sum_{h'=1}^{h-1} \widetilde{\ell}^{\,t}_{h'}
  \leq \eta H\,,
\end{align*}
where the last inequality follows from $\mu_{1:h}^t \widetilde{\ell}^{\,t}_{h'} = \mu_{1:h}^t (1-r^t_h) / (\mu_{1:h'}^t + \gamma) \leq 1$ for any $h' \leq h$. Accordingly
\begin{align*}
  \D \pp{\mu^t \| \mu^{t+1}}
  \leq \eta^2 H \sum_{h=1}^H \widetilde{\ell}^{\,t}_h + \eta^2 \sum_{h=1}^H \mu_{1:h}^t \exp \pp{-\eta \sum_{h'=1}^{h-1} \widetilde{\ell}^{\,t}_{h'}} \pp{\widetilde{\ell}^{\,t}_h}^2
  \leq \eta^2 (1 + H) \sum_{h=1}^H \widetilde{\ell}^{\,t}_h\,,
\end{align*}
where again we used $\mu_{1:h}^t \widetilde{\ell}^{\,t}_{h'} \leq 1$. Recalling that $\widetilde{\ell}^{\,t}_h$ is non-zero only at $(x_h^t, a^t_h)$, we have that $\widetilde{\ell}^{\,t}_h = \sum_{x_h \in \cX_h, a_h \in \cA} \widetilde{\ell}^{\,t}_h (x_h, a_h)$. Thus we can use Lemma~\ref{lemma:generalized Neu lemma}, which implies 
\begin{align*}
  \eta^2 (1 + H) \sum_{t=1}^T \sum_{h = 1}^H  \widetilde{\ell}^{\,t}_h
  &\leq \eta^2 (1 + H) \sum_{t=1}^T \sum_{h = 1}^H   \sum_{x_h \in \cX_h, a_h \in \cA} \ell^t_h (x_h, a_h) + \frac{\eta^2 (1 + H) H \log(3H / \delta)}{2\gamma}
  \\
  &\leq \eta^2 (1 + H) TXA + \frac{\eta^2 (1 + H) H \iota}{2\gamma}\CommaBin
\end{align*}
where at the final line we loosened the bound by replacing $\log(3H / \delta)$ with $\iota$ to simplify the bound. This concludes the proof.
\end{proof}

\section{Conclusion}
We theoretically studied the problem of learning a NE of an IIG under a perfect-recall assumption. We provided the \ouralgo algorithm based on \OMD with the dilated entropy distance-generating function as a regularizer and implicit exploration for estimation of the losses. We proved a high-probability bound on the convergence rate to the NE of order $\tcO(X\sqrt{A}+Y\sqrt{B})/\sqrt{T})$ derived from a regret bound of order $\tcO(X\sqrt{AT})$ (for the max-player). Notably, the regret bound remains valid in the adversarial setting (where the opponent and the game are picked by an adversary). Furthermore, due to our choice of the regularizer, the updates of the policy (e.g., of the max-player) could be implemented with a time-complexity of $\cO(HA)$ per episode, which makes \ouralgo also computationally efficient. Precisely, the total time complexity (after $T$ episodes) is of order  $\cO(TH(A+B) + \min(TH,X) A + \min(TH,Y)B)$ while the space complexity is of order $\cO(\min(TH,X)A + \min(TH,Y)B)$. 

An interesting next direction of research would be to characterize the problem-independent optimal regret, e.g., for the max-player, in our setting.  We conjecture that it is of order $\tcO(\sqrt{XAT})$ even in the adversarial setting (where the opponent and the game are picked by an adversary). This would make our current bound to be loose by a factor $\sqrt{X}$.

\bibliography{library,new_ref}
\bibliographystyle{sty/plainnat}

\clearpage
\appendix

\section{Proof of the Folklore Theorem~\ref{theorem:nash and regret}}\label{sec:proof of folklore theorem}

In this appendix we provide a proof of Theorem~\ref{theorem:nash and regret}, which is a well-known folklore theorem, for completeness.

Recall that $V^{\mu, \nu}$ is linear in each realization plan $\mu$ and $\nu$. Therefore we have
$V^{\mu, \nu} = \angl{\mu, r^\nu},$ where we define $r^\nu_h (x_h, a_h) := \sum_{s_h \in x_h, b_h \in \cB} p_{1:h} (s_h) \nu_{1:h} (s_h, b_h) r_h (s_h, a_h, b_h).$ By definition, the regret of the min-player relative to some policy $\nu^\dagger \in \minpi$ is given as
\begin{align*}
  \regret^T_{\mathrm{min}} (\nu^\dagger)
  = \sum_{t=1}^T \pp{\angl{\mu^t, r^{\nu^t}} - \angl{\mu^t, r^{\nu^\dagger}} }
  = \sum_{t=1}^T \angl{\mu^t, r^{\nu^t}} - T \angl{\frac 1T \sum_{t=1}^T \mu^t, r^{\nu^\dagger}}\,.
\end{align*}
Therefore if
\begin{equation}\label{eq:mixture distributions}
  \overline{\mu}_{1:h} (x_h, a_h) =\frac{1}{T} \sum_{t=1}^T \mu^t_{1:h} (x_h, a_h)
\end{equation}
holds for any $h \in [H]$ and any observation-action pair $(x_h, a_h) \in \cX_h \times \cA$, then
\begin{align*}
  \regret^T_{\mathrm{min}} (\nu^\dagger)
  = \sum_{t=1}^T \angl{\mu^t, r^{\nu^t}} - T \angl{\overline{\mu}, r^{\nu^\dagger}}
  = \sum_{t=1}^T \pp{ V^{\mu^t, \nu^t} - V^{\overline{\mu}, \nu^\dagger} }\,.
\end{align*}
A similar result holds for the regret of the max-player, and we have
\begin{align*}
  &\max_{\mu^\dagger \in \maxpi} V^{\mu^\dagger, \overline{\nu}} - \min_{\nu^\dagger \in \minpi} V^{\overline{\mu}, \nu^\dagger}
  \\
  &=  \max_{\mu^\dagger \in \maxpi} \frac 1T \sum_{t=1}^T  \pp{V^{\mu^\dagger, \nu^t} -  V^{\mu^t, \nu^t}}   - \min_{\nu^\dagger \in \minpi} \frac 1T \sum_{t=1}^T \pp{  V^{\mu^t, \nu^\dagger} - V^{\mu^t, \nu^t}}
  \\
  &= \frac 1T \pp{\max_{\mu^\dagger \in \maxpi} \regret^T_{\mathrm{max}} \pp{\mu^\dagger} +  \max_{\nu^\dagger \in \minpi} \regret^T_{\mathrm{min}} \pp{\nu^\dagger}}
  \leq \varepsilon.
\end{align*}
Thus $(\overline{\mu}, \overline{\nu})$ is an $\varepsilon$-NE.

We now prove Equation~\ref{eq:mixture distributions} by induction over $h$. This property is obviously true for $h=1$ from the definition of the average profile \eqref{eq:mixture policy}. Now assume Equation~\ref{eq:mixture distributions} holds for any observation-action pair $(x_{h'},a_{h'})$ of depth $h'< h$. Consider an observation $x_h \in \cX_h$ of depth $h$. Write $(x_{h-1}, a_{h-1})$ its immediate predecessor.
Then from the definition of $\overline{\mu}$ we have, for any $a_h \in \cA$, 
\begin{align*}
  \overline{\mu}_{1:h} (x_h, a_h) 
  &= \overline{\mu}_{1:h-1} (x_h) \overline{\mu}_{h} (a_h|x_h) \\
  &= \overline{\mu}_{1:h-1} (x_{h-1}, a_{h-1})
  \frac{\sum_{t=1}^T \mu^t_{1:h} (x_h, a_h)}{\sum_{t=1}^T \mu^t_{1:h-1} (x_h)}\\
  &= \frac 1T\sum_{t=1}^T \mu^t_{1:{h-1}} (x_{h-1}, a_{h-1})
  \frac{\sum_{t=1}^T \mu^t_{1:h} (x_h, a_h)}{\sum_{t=1}^T \mu^t_{1:h-1} (x_h)}\\
  &= \frac 1T\sum_{t=1}^T \mu^t_{1:{h-1}} (x_{h})
  \frac{\sum_{t=1}^T \mu^t_{1:h} (x_h, a_h)}{\sum_{t=1}^T \mu^t_{1:h-1} (x_h)}\\
  &= \frac 1T \sum_{t=1}^T \mu^t_{1:h} (x_h, a_h),
\end{align*}
Thus Equation~\ref{eq:mixture distributions} holds for any $h \in [H]$, and we conclude the proof.

\section{Details of Efficient Implementation (Section~\ref{sec:implementation})}\label{sec:details of implementation}

In this appendix we prove that the update \eqref{eq:OMD update} corresponds to the policy update \eqref{eq:policy update}, which is shown here for convenience.
\begin{align*}
  \mu_h^{t+1} (a_h | x_h^t) = \mu^t_h (a_h | x_h^t) \exp \pp{ \I_{\{a_h=a_h^t\}} \pp{ - \eta \widetilde{\ell}^{\,t}_h (x_h^t, a_h) + \log Z_{h+1}^t} - \log Z^t_h }\,,
\end{align*}
where
\begin{align*}
  Z^t_h
  &:= \sum_{a_h \in \cA} \mu^t_h (a_h | x_h^t) \exp \pp{ \I_{a_h=a_h^t} \pp{ - \eta \widetilde{\ell}^{\,t}_h (x_h^t, a_h) + \log Z_{h+1}^t} }
  \\
  &= 1 - \mu^t_h (a_h^t | x_h^t) + \mu^t_h (a_h^t | x_h^t) \exp \pp{ - \eta \widetilde{\ell}^{\,t}_h (x_h^t, a_h) + \log Z_{h+1}^t }
\end{align*}
with $Z^t_{H+1} := 1$. Note that no policy updates occur at unvisited information sets.

We prove the correspondence by induction on $h$. Recall that $\widetilde{\ell}^{\,t}$ is non-zero only at visited information sets and actions $(x_h^t, a_h^t)_{h \in [H]}$. Therefore
\begin{align*}
  \eta \angl{\mu, \widetilde{\ell}^{\,t}} + \D \pp{ \mu \mbb \mu^t}
  = \sum_{h = 1}^H \pp{ \eta \mu_{1:h} (x_h^t, a_h^t) \widetilde{\ell}^{\,t}_h (x_h^t, a_h^t) + \sum_{x_h \in \cX_h} \mu_{1:h-1} (x_h) \KL \pp{\mu_h  \mbb \mu_h^t }(x_h^t) }\,,
\end{align*}
where $\KL \pp{\mu_h  \mbb \mu_h^t }(x)$ is a shorthand notation for Kullback-Leibler divergence $\KL \pp{\mu_h (\cdot|x) \mbb \mu_h^t (\cdot|x) }$. Because it suffices to optimize $\mu$ at visited information sets, we may focus on terms involving them. Accordingly to find $\mu^{t+1}$ we need to minimize
\begin{align*}
  \fL \pp{\mu_1, \ldots, \mu_H} := \sum_{h = 1}^H \mu_{1:h-1} (x_h^t) \pp{ \eta  \mu_h (a_h^t | x_h^t) \widetilde{\ell}^{\,t}_h (x_h^t, a_h^t) + \KL \pp{\mu_h \mbb \mu_h^t} (x_h) }
\end{align*}
with respect to $\mu$. For $h=H$ it is straightforward to deduce that
\begin{align*}
  \mu_H^{t+1} (a_H | x_H^t) = \mu^t_H (a_H | x_H^t) \exp \pp{ - \eta \I_{\{a_H=a_H^t\}} \widetilde{\ell}^{\,t}_H (x_H^t, a_H) - \log Z^t_H }\,.
\end{align*}

Assume that the claim holds up to step $h+1$. Then for $\mu$ such that $\mu_{h'} = \mu_{h'}^{t+1}$ for $h'>h$ we have 
\begin{align*}
  &\fL \pp{\mu_1, \ldots, \mu_H}
  \\
  &= \sum_{h' = 1}^h \mu_{1:h'-1} (x_{h'}^t) \pp{ \eta  \mu_{h'} (a_{h'}^t | x_{h'}^t) \widetilde{\ell}^{\,t}_{h'} (x_{h'}^t, a_{h'}^t) + \KL \pp{\mu_{h'} \mbb \mu_{h'}^t} (x_{h'}^t) }
  \\
  &\hspace{6em} + \sum_{h'=h+1}^H \mu_{1:h'-1} (x_{h'}^t) \pp{ \eta  \mu_{h'} (a_{h'}^t | x_{h'}^t) \widetilde{\ell}^{\,t}_{h'} (x_{h'}^t, a_{h'}^t) + \KL \pp{\mu_{h'} \mbb \mu_{h'}^t} (x_{h'}^t) }
  \\
  &= \sum_{h' = 1}^h \mu_{1:h'-1} (x_{h'}^t) \pp{ \eta  \mu_{h'} (a_{h'}^t | x_{h'}^t) \widetilde{\ell}^{\,t}_{h'} (x_{h'}^t, a_{h'}^t) + \KL \pp{\mu_{h'} \mbb \mu_{h'}^t} (x_{h'}^t) }
  \\
  &\hspace{6em} + \sum_{h'=h+1}^H \mu_{1:h'-1} (x_{h'}^t) \pp{ \mu_{h'} (a_{h'}^t | x_{h'}^t) \log Z_{h'+1}^t - \log Z_{h'}^t}
  \\
  &= \sum_{h' = 1}^h \mu_{1:h'-1} (x_{h'}^t) \pp{ \eta  \mu_{h'} (a_{h'}^t | x_{h'}^t) \widetilde{\ell}^{\,t}_{h'} (x_{h'}^t, a_{h'}^t) + \KL \pp{\mu_{h'} \mbb \mu_{h'}^t} (x_{h'}^t) } - \mu_{1:h} (x_{h+1}^t) \log Z_{h+1}^t
  \\
  &= \sum_{h' = 1}^{h-1} \mu_{1:h'-1} (x_{h'}^t) \pp{ \eta  \mu_{h'} (a_{h'}^t | x_{h'}^t) \widetilde{\ell}^{\,t}_{h'} (x_{h'}^t, a_{h'}^t) + \KL \pp{\mu_{h'} \mbb \mu_{h'}^t} (x_{h'}^t) }
  \\
  &\hspace{6em} + \mu_{1:h-1} (x_h^t) \pp{ \mu_h (a_h^t | x_h^t) \pp{ \eta \widetilde{\ell}^{\,t}_h (x_h^t, a_h^t) - \log Z_{h+1}^t } + \KL \pp{\mu_h \mbb \mu_h^t} (x_h^t)}\,.
\end{align*}
Therefore we deduce that
\begin{align*}
  \mu_h^{t+1} (a_h | x_h^t) = \mu^t_h (a_h | x_h^t) \exp \pp{ \I_{\{a_h=a_h^t\}} \pp{ - \eta \widetilde{\ell}^{\,t}_h (x_h^t, a_h) + \log Z_{h+1}^t} - \log Z^t_h }\,.
\end{align*}
This concludes the proof.

\section{Efficient Computation of the Average Policy}\label{sec:efficient computation of average policy}

In this appendix we explain how to efficiently compute the average policy in Theorem~\ref{theorem:nash and regret}.

We define $\tau_h^t: \cX \rightarrow \{0\} \cup \N$ by
\begin{align*}
  \tau_h^t (x) := \max \pp{ \{ 0 \} \cup \{ 1\leq k < t :\ x_h^k = x,\ k \in \N \} }.
\end{align*}
In other words, $\tau_h^t (x)$ is an index of an episode at which $x$ has been visited last time before $t$ (if it has been visited, otherwise returns $0$). Further we define $\muring_{1:h}^t: \cX_h \times \cA \rightarrow [0, \infty)$ for each $h \in [H]$ by
\begin{gather*}
  \muring_{1:h}^t (x_h, a_h) := \sum_{u=1}^t \mu_{1:h}^u (x_h, a_h).
\end{gather*}
Using this function, we can compute the average policy since for any $t$
\begin{align*}
    \frac{\sum_{u=1}^t \mu_{1:h}^u (x_h, a_h)}{\sum_{u=1}^t \mu_{1:h-1}^u (x_h)} = \frac{\muring_{1:h}^t (x_h, a_h)}{\sum_{a_h' \in \cA} \muring_{1:h}^t (x_h, a_h')}.
\end{align*}
Hence, we can compute the average policy after learning by using $\muring_{1:h}^T$.

Interestingly $\muring_{1:h}^t (x_h, a_h)$ can be computed while traversing a game tree by only using $\mu^t$ and a value available at the last time visitation to $x_h$. To see this, consider a fixed $(x_h, a_h) \in \cX_h \times \cA$ with $h > 1$ and let $\tau := \tau_h^t (x_h)$ for brevity. Since the policy does not change between $\tau + 1$ and $t$, we have that
\begin{align*}
    \muring_{1:h}^t (x_h,a_h)
    &= \sum_{u=1}^t \mu_{1:h}^u (x_h,a_h)
    \\
    &= \sum_{u=1}^\tau \mu_{1:h}^u (x_h,a_h) + \sum_{u=\tau+1}^t \mu_{1:h-1}^u (x_h) \underbrace{\mu_h^u(a_h|x_h)}_{=\mu_h^t (a_h|x_h)}
    \\
    &= \muring_{1:h}^\tau (x_h,a_h) + \pp{ \sum_{u=\tau+1}^t \mu_{1:h-1}^u (x_h) } \mu_h^t (a_h|x_h)
        \\
    &= \muring_{1:h}^\tau (x_h, a_h) + \pp{ \muring_{1:h-1}^t (x_{h-1}, a_{h-1}) -  \muring_{1:h-1}^\tau (x_{h-1}, a_{h-1}) } \mu_h^t (a_h|x_h),
\end{align*}
where $(x_{h-1}, a_{h-1})$ is a unique predecessor of $x_h$. Therefore we can compute the average policy while traversing a game tree by using $\mu^t$ and $\muring_{1:h-1}^{\tau}(x_{h-1}, a_{h-1})$ stored at the last-visitation to $x_h$. For $h=1$, a similar result holds by reading $\mu_{1:h-1}^u (x_h)$ as $1$.

Therefore, once the learning ends, we can compute $\muring_{1:h}^T (x_h,a_h)$ for all visited information sets and actions, using stored transition data and $\muring_{1:h}^{\tau}$. (At non-visited information sets, the average policy chooses actions uniformly, and thus, no computation is required.) For a full pseudocode, see Algorithm~\ref{alg:memory efficient ouralgo}.

\section{Practical Implementation of \texorpdfstring{\ouralgo}{IXOMD}}\label{sec:memory efficient implementation}

In this appendix we provide a pseudocode for \efficientouralgo, a practical version of \ouralgo. Without loss of generality, we assume that $\cA = \{1, \ldots, A\}$. We use Python-like list $\List$, dictionary $\Dict$, and Set $\Set$ objects  (but we assume that the index of a list starts from $1$). We also follow Python-like notations.

Algorithm~\ref{alg:policy} is a pseudocode for a memory-efficient implementation of the policy. It only stores action probabilities for observed information sets. We note that $\policyalgo\algofont{.batchUpdate}$, which is called once per episode, has $\cO(HA)$ time-complexity.

Algorithm~\ref{alg:memory efficient ouralgo} is a pseudocode for \efficientouralgo. Line~6 to 28 correspond to the learning of the policy. As noted in the last paragraph, $\policyalgo\algofont{.batchUpdate}$ is called once per episode, and thus, the total time-complexity for the learning of the policy is $\cO(THA)$. While traversing a game tree, we also perform the update of $\muring^t$, which is used to compute the average policy as described in Appendix~\ref{sec:efficient computation of average policy}. For one traversal, this update requires $\cO(THA)$ time-complexity in total. Line~29 to the end of the code correspond to the computation of $\overline{\mu}$ defined in Theorem~\ref{theorem:nash and regret}. This part has $\cO(\min(TH, X) A)$ time-complexity. As for the space-complexity, $\algofont{muDot}$ requires the largest memory space, which is $\cO(\min(TH, X) A)$.

\begin{algorithm}[h!]
  \SetAlgoLined
  \DontPrintSemicolon
  \SetKwFunction{init}{\char`_\char`_init\char`_\char`_}
  \SetKwProg{Fn}{function}{:}{}
  \Fn{\init{}}{
    $\algofont{knownObs} = \Set()$.\;
    $\algofont{actionProbas} = \Dict()$.\;
  }
  \hrulefill\\
  \SetKwFunction{proba}{getActionProba}
  \Fn{\proba{$x$, $a$}}{
    $p = \algofont{actionProbas}[(x, a)]$ if $x$ in $\algofont{knownObs}$ else $1 / A$.\;
    \Return{$p$}.\;
  }
  \hrulefill\\
  \SetKwFunction{probas}{getActionProbas}
  \Fn{\probas{$x$}}{
    $\algofont{probas} = \List()$.\;
    \For{$a = 1,\ldots,A$}{
      $\algofont{probas.append}(\algofont{getActionProba}(x, a))$.\;
    }
    \Return{$\algofont{probas}$}
  }
  \hrulefill\\
  \SetKwFunction{update}{update}
  \Fn{\update{$x, a, p$}}{
    $\algofont{actionProbas}[(x, a)] = p$.\;
    $\algofont{knownObs.add}(x)$.\;
  }
  \hrulefill\\
  \SetKwFunction{batchupdate}{batchUpdate}
  \Fn{\batchupdate{$\algofont{traj}$}}{
    $\mu_{1:0} = 1$.\;
    \For{$h=1, \ldots, H$}{
      $x_h, a_h, r_h$ = $\algofont{traj}$[$h$].\;
      $\mu_h = \algofont{actionProbas}[(x_h, a_h)]$.\;
      $\mu_{1:h} = \mu_{1:h-1} \mu_h$.\;
    }
    $Z_{H+1} = 1$.\;
    \For{$h=H, \ldots, 1$}{
      $x_h, a_h, r_h$ = $\algofont{traj}[h]$.\;
      $\widetilde{\ell}_h = (1 - r_h) / (\mu_{1:h} + \gamma)$.\;
      $Z_h = 1 - \mu_h + \mu_h \exp ( - \eta \widetilde{\ell}_h + \log Z_{h+1} )$.\;
      $\algofont{probas} = \algofont{getActionProbas}(x_h)$.\;
      \For{$a = 1, \ldots, A$}{
        $\algofont{update}(x_h, a, \algofont{probas}[a]\, \exp ( \I_{a=a_h} (- \eta \widetilde{\ell}_h + \log Z_{h+1}) - \log Z_h))$.\;
      }
    }
  }
  \caption{\policyalgo}\label{alg:policy}
\end{algorithm}

\begin{algorithm}[t!]
  \SetAlgoLined
  \DontPrintSemicolon
  \KwIn{IX hyper-parameter $\gamma \in (0, \infty)$ and \OMD's learning rate $\eta \in (0, \infty)$.}
  \KwOut{A near-NE policy for the max-player.}
  $\algofont{pred} = \List()$, $\algofont{muDot} = \List()$, $\algofont{lastMuDotX} = \List()$.\;
  \For{$h=1,\ldots, H$}{
    \tcp{This initialization can be done later while playing the game.}
    $\algofont{pred.append}(\Dict())$, $\algofont{muDot.append}(\Dict())$, $\algofont{lastMuDotX.append}(\Dict())$.\;
  }
  $\algofont{policy} = \policyalgo()$, $\algofont{lastIdx} = \Dict()$, $\algofont{knownObs} = \Set()$.\;
  \tcp{Learn policies playing the game.}
  \For{$t=1, \ldots, T-1$}{
    $\algofont{traj} = \List$(), $x_0^t = \varnothing$, $a_0^t = \varnothing$.\;
    \For{$h = 1, \ldots, H$}{
      Observe $x_h^t$ and compute $\algofont{probas} = \algofont{policy.getActionProbas}(x_h^t)$.\;
      Execute $a_h^t$ sampled from $\algofont{probas}$, receive $r_h^t$\,, and $\algofont{traj.append}((x_h^t, a_h^t, r_h^t))$.\;
      \If{$x_h^t \notin \algofont{knownObs}$}{
        \For{$a=1, \ldots, A$}{
          $\algofont{muDot}[h][(x_h^t, a)] = 0$.\;
        }
        $\algofont{lastMuDotX}[h][x_h^t] = 0$, $\algofont{pred}[h][x_h^t] = (x_{h-1}^t, a_{h-1}^t)$, $\algofont{knownObs.add}(x_h^t)$.\;
      }
      \uIf{$h=1$}{
        $\algofont{diff} = t - \algofont{lastMuDotX}[h][x_h^t]$, $\algofont{lastMuDotX}[h][x_h^t] = t$.\;
      } \Else {
        $\algofont{diff} = \algofont{muDot}[h-1][(x_{h-1}^t, a_{h-1}^t)] - \algofont{lastMuDotX}[h][x_h^t]$.\;
        $\algofont{lastMuDotX}[h][x_h^t] = \algofont{muDot}[h-1][(x_{h-1}^t, a_{h-1}^t)]$.\;
      }
      \For{$a=1, \ldots, A$}{
        $\algofont{muDot}[h][(x_h^t, a)] \pluseq \algofont{diff} \times \algofont{probas}[a]$.\;
      }
    }
    $\algofont{policy.batchUpdate}(\algofont{traj})$.\;
  }
  \tcp{Compute the average policy.}
  $\algofont{averagePolicy} = \policyalgo()$.\;
  \For{$h=1, \ldots, H$}{
    \tcp{Size of $\algofont{pred}[h]\algofont{.keys}()$ is $\min(T, |\cX_h|)$.}
    \For{$x_h \in \algofont{pred}[h]\algofont{.keys}()$}{
      $x_{h-1}, a_{h-1} = \algofont{pred}[h][x_h]$.\;
      \uIf{$h=1$}{
        $\algofont{diff} = T - \algofont{lastMuDotX}[h][x_h]$.\;
      } \Else {
        $\algofont{diff} = \algofont{muDot}[h-1][(x_{h-1}, a_{h-1})] - \algofont{lastMuDotX}[h][x_h]$.\;
      }
      \For{$a=1, \ldots, A$}{
        $\algofont{muDot}[h][(x_h, a)] \pluseq \algofont{diff} \times \algofont{probas}[a]$.\;
      }
      $\algofont{sum} = \sum_{a' \in \cA} \algofont{muDot}[h][(x_h, a')]$.\;
      \For{$a=1, \ldots, A$}{
        $p = \algofont{muDot}[h][(x_h, a)] / \algofont{sum}$.\;
        $\algofont{averagePolicy.update}(x_h, a, p)$.\;
      }
    }
  }
  \Return{Policy $\algofont{averagePolicy}$ the average $\overline{\mu}$ of $\mu^1,\ldots,\mu^T$ defined in Theorem~\ref{theorem:nash and regret}.}
  \caption{\efficientouralgo for the Max-player}\label{alg:memory efficient ouralgo}
\end{algorithm}

\end{document}